\documentclass[a4paper,conference]{IEEEtran}
\IEEEoverridecommandlockouts
\usepackage{cite}
\usepackage{amsmath,amssymb,amsfonts}
\usepackage{amsthm}
\usepackage{algorithmic}
\usepackage{colortbl}
\definecolor{secgray}{gray}{0.93}
\usepackage{booktabs}
\usepackage{tabularx}
\usepackage{graphicx}
\usepackage{textcomp}
\usepackage{algorithm}
\usepackage{xcolor}
\usepackage{bm}
\usepackage[normalem]{ulem}
\definecolor{negred}{RGB}{180,30,30}

\newtheorem{proposition}{Proposition}
\newtheorem{remark}{Remark}

\def\BibTeX{{\rm B\kern-.05em{\sc i\kern-.025em b}\kern-.08em
    T\kern-.1667em\lower.7ex\hbox{E}\kern-.125emX}}

\definecolor{darkgreen}{RGB}{0,100,0}

\newcommand{\Wspace}{\mathcal{X}^{o}}

\newcommand{\wscen}[1]{x^{o,#1}_{0:H}}
\newcommand{\Wtraj}{\boldsymbol{X}^{o}_{0:H}}
\newcommand{\zij}[2]{z^{#1}_{#2}}
\newcommand{\rhobar}[1]{\bar{\rho}_{#1}}

\begin{document}

\title{Autonomous Driving with Priority-Ordered STL Specifications Under Multimodal Uncertainty}

\author{Taha Bouzid, Shuhao Qi, Mircea Lazar, and Sofie Haesaert}

\maketitle
\begingroup
\renewcommand\thefootnote{}
\footnote{This research has received funding from the European Union under the
Horizon Europe Grant Agreement AIGGREGATE, no. 101202457. All
authors are with the Control Systems Group, Department of Electrical
Engineering, Eindhoven University of Technology, Eindhoven, The Netherlands.
Email: \texttt{\{t.bouzid, s.qi, m.lazar, s.haesaert\}@tue.nl}}
\endgroup

\begin{abstract}
Autonomous vehicles must plan trajectories that satisfy multiple requirements,
such as safety, traffic-rule compliance, and passenger comfort.
However, in safety-critical scenarios, it is not always possible to satisfy all requirements simultaneously, necessitating their prioritization based on importance. At the same time, the uncertainty in the predicted trajectories of surrounding road users, such as other vehicles and pedestrians, must be explicitly accounted for. In this work, we propose an uncertainty-aware trajectory planning framework that incorporates a predefined  priority ordering over Signal Temporal Logic (STL) specifications and preserves the induced lexicographic ordering under multimodal uncertainty. We implement this formulation with Model Predictive Path Integral (MPPI) control and demonstrate the effectiveness of our method on simulation scenarios, showing that our framework efficiently handles conflicting objectives under realistic multimodal uncertainty.
\end{abstract}

\section{Introduction}
Autonomous vehicles (AVs) must make real-time planning decisions based on uncertain predictions of the motion of other road users. These predictions are often multimodal \cite{cui2019multimodal, qi2025situation}. A scenario with such a multimodal prediction is given in Fig.~\ref{fig:highway_scen}: the green car needs to account for two possible predicted motions of the blue car, each with a different
probability. 
At the same time, its planned trajectory must satisfy multiple criteria, such as safety requirements, traffic-rule compliance, and passenger comfort~\cite{9fecae2207c84bcaaaad1358bee7ae55}. In realistic situations, these objectives cannot always be satisfied simultaneously across all modes. Consider the scenario in Fig.~\ref{fig:highway_scen}: the ego vehicle must plan under two equally likely future modes of the other agent, namely keeping its lane or cutting in, with only two candidate trajectories available. Trajectory $T_2$ satisfies both the safety and the passenger-comfort rules in the lane-keeping mode but risks a collision in the cut-in mode, whereas $T_1$ satisfies the safety rule in both modes at the cost of violating the  passenger-comfort rule in each. In such safety-critical scenarios, $T_1$ is clearly preferable, since it guarantees satisfaction of the highest-priority safety rule across all possible futures. This illustrates that, under multimodal uncertainty, it is essential to balance multiple objectives according to their priority order.
\begin{figure}[htbp]
	\centering
        \includegraphics[width=\linewidth]{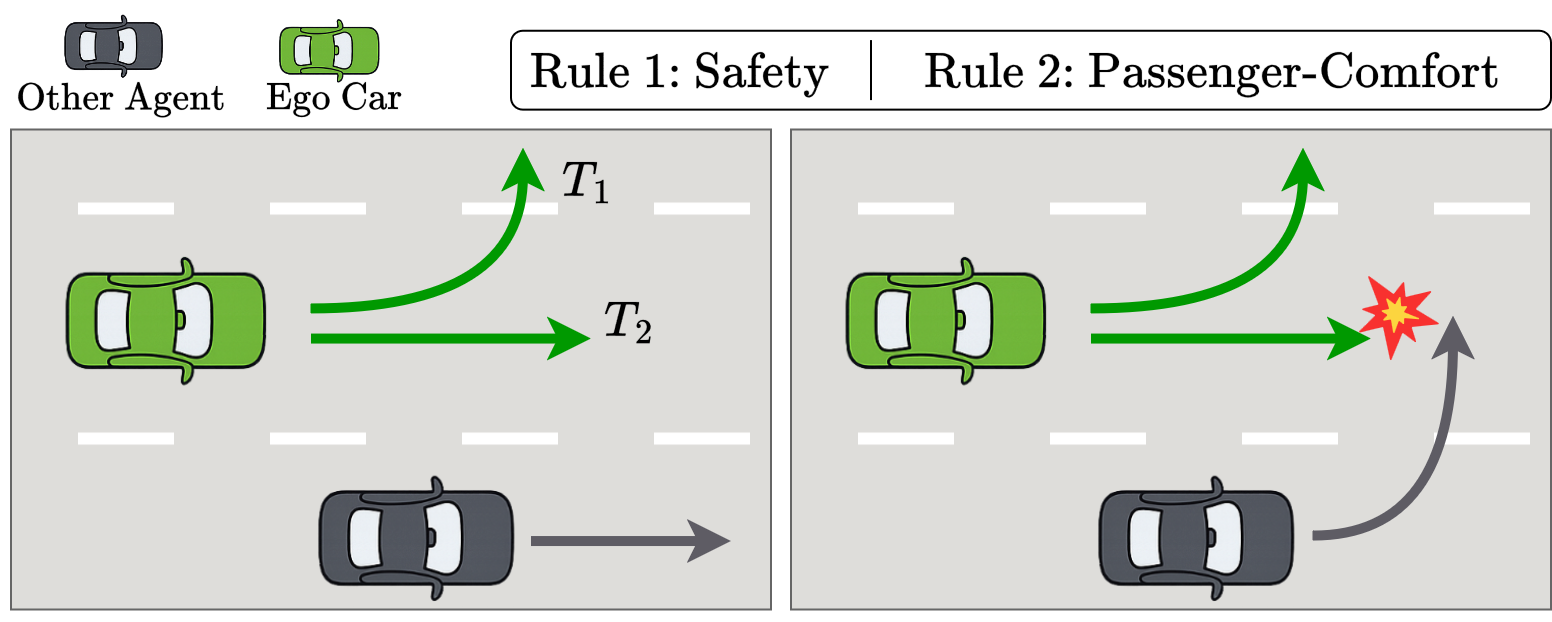}
\caption{Illustrative trajectory-planning scenario under multimodal uncertainty in the other agent's future motion, with its lane-keeping mode shown on the left and cut-in mode on the right. The ego vehicle considers two candidate trajectories, $T_1$ and $T_2$, and must select the one that respects the predefined rule ordering under all modes.\vspace{-5.5mm}}
        \label{fig:highway_scen}
\end{figure}

To verify whether a trajectory satisfies its planning objectives, such objectives are generally translated into formal specifications encoding safety and 
traffic rules using signal temporal logic 
(STL)~\cite{maler2004monitoring,donze2010robust}, as 
in~\cite{maierhofer2020interstate,maierhofer2022intersection,
9fecae2207c84bcaaaad1358bee7ae55}, or with linear temporal logic 
(LTL)~\cite{risk_Qi,Intention-Aware}.
STL in particular offers the advantage of robustness 
semantics~\cite{donze2010robust}, enabling the reformulation of the 
constrained planning problem as an optimization problem that maximizes 
the degree of satisfaction or violation of the specifications.
For deterministic planning problems, the strict priority among traffic rules is encoded in~\cite{censi2019rulebooks} as a lexicographic hierarchy, in which any violation of a higher-priority STL rule dominates any combination of satisfactions and violations at lower priority. Building on this formulation, an equivalent differentiable rank-preserving reward that captures the same hierarchy is constructed in~\cite{veer2023receding} and embedded within a gradient-based receding-horizon planner.
%
Halder and Althoff~\cite{halder2023lexicographic} embedded this priority
structure into a lexicographic mixed-integer program over STL
constraints, and their follow-up work~\cite{halder2025preordered}
proposed a sampling-based planner with preordered objectives that avoids
scalarization.
 However, these methods address priority-ordered rule satisfaction only in a deterministic setting, which is inadequate when the motion of surrounding objects is uncertain. Under multimodal traffic predictions, rule satisfaction is inherently combinatorial: a trajectory may satisfy the highest-priority rule under one predicted mode of opponents but violate it under another predicted mode. Closing this gap therefore requires reasoning about rule satisfaction over the full distribution of future traffic evolution.
 Furthermore, to capture uncertainty arising from complex tasks such as those involving human road users, risk-aware formulations have also been proposed. In~\cite{risk_Qi}, a risk metric for LTL specifications is introduced that integrates the severity and timing of uncertain events. In~\cite{lindemann2022riskstl}, risk metrics such as the Conditional Value-at-Risk (CVaR) are applied to STL specifications to capture the expected severity of rule violations in the tail of the loss distribution. However, none of these works incorporates a strict lexicographic priority ordering over STL specifications.



We present a novel trajectory-planning framework that accounts for priority-ordered STL rules under multimodal uncertainty in the motion of surrounding objects. The resulting non-smooth, non-convex planning problem is solved using
a Model Predictive Path Integral (MPPI)-based receding-horizon planner. In detail, our contributions are:
\begin{itemize}
    \item A stochastic extension of lexicographically ordered STL planning using risk-aware robustness satisfaction that preserves strict priority relations under uncertainty.
    \item An MPPI-based planner for optimizing non-smooth ordered STL objectives in multimodal traffic scenarios.
\end{itemize}

\section{Preliminaries and Problem Statement}

\subsection{System Definition}
We refer to the AV under control as the \emph{ego vehicle}, whose state and control input at step $k$ are denoted by $x^e_k \!\in\! \mathcal{X} \!\subseteq\! \mathbb{R}^n$ and $u_k \!\in\! \mathcal{U} \!\subseteq\! \mathbb{R}^m$, respectively. The ego dynamics evolve according to the deterministic, discrete-time nonlinear system
\begin{equation}\label{eq:ego_dynamics}
x^e_{k+1} = f(x^e_k, u_k).
\end{equation}
For an admissible initial state $x^e_0$ and input sequence $u_{0:H-1}\!:=\!(u_0,\ldots,u_{H-1})$ over the planning horizon $H$, the resulting state trajectory is $x^e_{0:H}\!:=\!(x^e_0,\ldots,x^e_H)$. Analogously, let $x^{o}_k \!\in\! \Wspace \!\subseteq\! \mathbb{R}^q$ denote the state of a surrounding object at step $k$.
Since the future motion of the surrounding object is stochastic, we model it
over a probability space $(\Omega, \mathcal{A}, \mathbb{P})$, where $\Omega$ is
the sample space, $\mathcal{A}$ is a $\sigma$-algebra of measurable events, and
$\mathbb{P} \colon \mathcal{A} \!\to\! [0,1]$ is a probability measure. Specifically,
its future trajectory over the planning horizon is modeled as a random trajectory
variable $\Wtraj : \Omega \!\to\! (\Wspace)^{H+1}$, where each realization
$\Wtraj(\omega) \!=\! (x^{o}_0(\omega), x^{o}_1(\omega), \dots, x^{o}_H(\omega))$,
$\omega \!\in\! \Omega$, corresponds to one possible future trajectory of the
surrounding object. For notational clarity, we present the formulation for a single surrounding object; it extends easily to the multi-object case. The state of the scenario is composed of the state of the ego vehicle and that of the surrounding object, and is denoted as
$s_k\!:=\!(x_k^e,x_k^o)\!\in\!\mathbb{R}^{n+q}$ at discrete step $k$.
Similarly, we can define a full realization of the scenario over the horizon $\{0,1,\dots,H\}$ by $s_{0:H}\!:=\!(s_0,s_1,\dots,s_H)$. Note that this realization is therefore also implicitly a function of $\omega$.
In this work, we consider \emph{multimodal} stochastic uncertainty over the future motion of the surrounding object, which is therefore more general than the unimodal assumption adopted in prior work~\cite{engelaar2024riskmpc}.

\subsection{Signal Temporal Logic (STL)}
\label{subsec:stl}
In this work, we
consider a future-time, discrete-time fragment of STL~\cite{maler2004monitoring}. 
We consider time intervals $I\!\subseteq\!\mathbb{Z}_{\!\ge\! 0}$. For integers
$a,b\!\in\!\mathbb{Z}_{\!\ge\! 0}$ with $b\!\ge\! a$, let
$I\!:=\!\{a,a+1,\dots,b\}$; then we define
$k+I\!:=\!\{k+a,k+a+1,\dots,k+b\}$ for $k\!\in\! \mathbb{Z}_{\!\ge\! 0}$. The STL syntax 
over the
signal $s:\mathbb{Z}_{\!\ge\! 0}\!\rightarrow\! \mathbb R^{n+q}$ 
is defined as follows:
\begin{equation}
\varphi:=\top\mid\mu\mid\neg\varphi\mid\varphi_1\wedge\varphi_2\mid\varphi_1 \mathsf{U}_I\varphi_2,
\end{equation}
where $\mu\!:=\!(l(s_k)\!\ge\! 0)$ is a predicate with $l:\mathbb{R}^{n+q}\!\to\!\mathbb{R}$,
and $\neg$, $\wedge$, and $\mathsf{U}_I$ denote negation, conjunction, and bounded
until, respectively. The derived operators are $\varphi_1\vee\varphi_2
\!:=\!\neg(\neg\varphi_1\wedge\neg\varphi_2)$ (\emph{disjunction}),
$F_I\varphi\!:=\!\top \mathsf{U}_I\varphi$ (\emph{eventually}), and
$G_I\varphi\!:=\!\neg F_I\neg\varphi$ (\emph{always}).
Satisfaction of $\varphi$ by $s$ at step $k$ is denoted by $s\models_k\varphi$.
The quantitative semantics of STL based on their
robustness~\cite{donze2010robust} are recursively defined by
\begin{equation}
\begin{aligned}
\rho(\top,s,k)      &:= +\infty,\\
\rho(\mu,s,k)       &:= l(s_k),\\
\rho(\neg\varphi,s,k)               &:= -\rho(\varphi,s,k),\\
\rho(\varphi_1\wedge\varphi_2,s,k) &:= \min\bigl(\rho(\varphi_1,s,k),\,\rho(\varphi_2,s,k)\bigr),\\
\rho(\varphi_1\vee\varphi_2,s,k)   &:= \max\bigl(\rho(\varphi_1,s,k),\,\rho(\varphi_2,s,k)\bigr),\\[2pt]
\rho(\varphi_1 \mathsf{U}_I \varphi_2,s,k) &:= \max_{k'\in k+I}\min\bigl(
                                         \rho(\varphi_2,s,k'),\\
    &\quad \hphantom{:={}}
                                         \min_{k''\in[k,k')}
                                         \rho(\varphi_1,s,k'')\bigr),\\[2pt]
\rho(G_I\varphi,s,k)               &:= \min_{k'\in k+I}\rho(\varphi,s,k'),\\
\rho(F_I\varphi,s,k)               &:= \max_{k'\in k+I}\rho(\varphi,s,k').
\end{aligned}
\label{eq:stl_robustness_semantics}
\end{equation}

Based on the STL robustness~\cite{donze2010robust}, we say that the signal $s$ satisfies $\varphi$, that is, $s\vDash\varphi$, if $\rho(\varphi,s,0)>0$ and $s\not\vDash\varphi$ if $\rho(\varphi,s,0)<0$.   

Since $\varphi$ is bounded-time, the robustness $\rho(\varphi,s,0)$
depends only on the finite prefix $s_{0:H}$ whenever
$H \!\ge\! \mathrm{hrz}(\varphi)$, where $\mathrm{hrz}(\varphi)$ is the
horizon of $\varphi$, i.e., the minimum signal length required to evaluate
its satisfaction~\cite{lindemann2022riskstl}.
We therefore work with the truncated signal $s_{0:H}$ in what follows.

\subsection{Conditional Value-at-Risk (CVaR)}
\label{subsec:risk_metrics}

The robustness defined in~\eqref{eq:stl_robustness_semantics} depends on
the composed signal $s_{0:H}$.
Since $x^e_{0:H}$ is a deterministic decision variable while $\boldsymbol{X}^o_{0:H}$
is a random variable, the composed signal $s_{0:H}$ is itself a random variable
for any fixed ego trajectory, with realization
$s_{0:H}(\omega) \!=\! \bigl((x^e_0, x^o_0(\omega)),\dots,(x^e_H, x^o_H(\omega))\bigr)$
with $\omega \!\in\! \Omega$.
Following~\cite{lindemann2022riskstl}, we define the \emph{loss random variable} associated with specification $\varphi$ and ego trajectory $x^e_{0:H}$ as the negated STL robustness:
\begin{equation}
  \boldsymbol{Z}_{\varphi}(x^e_{0:H}) \colon \Omega \to \mathbb{R}, 
  \label{eq:loss_rv}
\end{equation}
with $\boldsymbol{Z}_{\varphi}(x^e_{0:H})(\omega) \!:=\! -\rho(\varphi,\, s_{0:H}(\omega),\, 0)$ for $\omega \!\in\! \Omega$. Since $\rho$ is positive under satisfaction and negative under violation, large positive values of $\boldsymbol{Z}_{\varphi}(x^e_{0:H})$ correspond to severe rule violations.
Evaluating $\boldsymbol{Z}_{\varphi}(x^e_{0:H})$ in expectation averages over all outcomes 
and can therefore mask rare but severe violations that are decisive in safety-critical 
planning~\cite{majumdar2020risk, rockafellar2000cvar}. To explicitly account for such tail events, we instead adopt a \emph{risk metric} $\mathcal{R}$, a mapping from real-valued random variables on $(\Omega, \mathcal{A}, \mathbb{P})$ to $\mathbb{R}$ whose level parameter controls how much weight is placed on the tail of the loss distribution. A common risk metric is the \emph{Value-at-Risk} ($\mathrm{VaR}_{\beta}$) at level $\beta \!\in\! [0,1)$, defined as the $\beta$-quantile of the loss distribution:
\[
  \mathrm{VaR}_{\beta}\!\left(\boldsymbol{Z}_{\varphi}(x^e_{0:H})\right)
    := \inf \left\{ \alpha \in \mathbb{R} :
       \mathbb{P}\!\left(\boldsymbol{Z}_{\varphi}(x^e_{0:H}) \le \alpha\right) \ge \beta
    \right\}.
\]
This risk metric captures the loss threshold exceeded only with probability $1-\beta$
but ignores the magnitude of losses beyond it.
In contrast, the \emph{Conditional Value-at-Risk} at level $\beta \!\in\! [0,1)$ is defined as
\[
  \mathrm{CVaR}_{\beta}\!\left(\boldsymbol{Z}_{\varphi}(x^e_{0:H})\right)
    := \inf_{\alpha \in \mathbb{R}}
       \!\big(\alpha + \frac{1}{1-\beta}
       \,\mathbb{E}\!\big[(\boldsymbol{Z}_{\varphi}(x^e_{0:H}) - \alpha)^{+}
       \big]\big),
\]
where $(\cdot)^{+} \!:=\! \max(\cdot,\, 0)$. CVaR averages the loss over the worst
$(1-\beta)$ fraction of outcomes
and hence penalizes the magnitude of tail losses; moreover, it is a coherent risk measure~\cite{rockafellar2000cvar}.
For notational brevity, we will use $\mathcal{R}$ interchangeably with $\mathrm{CVaR}_{\beta}$ in the remainder of this paper.

\subsection{Problem Statement}
\label{subsec:prob}
We say that an STL-rule lexicographic ordering over $\Phi =
(\varphi_1, \ldots, \varphi_N)$ is a strict priority relation if $\varphi_1 \succ \cdots \succ \varphi_N$,
where $\succ$ denotes ``higher priority than'', and if satisfying
a higher-priority rule strictly dominates any combination of
satisfactions and violations at lower priority levels. To assess rule satisfaction under the uncertainty in the predicted stochastic trajectory $\Wtraj$, each rule $\varphi_j \!\in\! \Phi$ is evaluated by its own risk metric $\mathcal{R}_j$ at confidence level $\beta_j$.

Given the ordered rule set $\Phi$, the per-rule risk metrics $\mathcal{R}_j$, the ego dynamics~\eqref{eq:ego_dynamics} initialized at $x^e_0$, and the predicted stochastic trajectory $\Wtraj$, our goal is to synthesize a receding-horizon trajectory planner for the ego vehicle that respects the lexicographic priority $\varphi_1 \succ \cdots \succ \varphi_N$ in the risk-aware sense.



\section{Planning under rule ordering and uncertainty} 
\label{sec:Planning under rule ordering and uncertainty}

\subsection{Risk-Extended Satisfaction}
As introduced in Section~\ref{subsec:risk_metrics}, for a fixed ego trajectory
$x^e_{0:H}$, the loss random variable $\boldsymbol{Z}_{\varphi_j}(x^e_{0:H})$
defined in~\eqref{eq:loss_rv} captures the uncertainty in the satisfaction
of rule $\varphi_j$ induced by the stochastic surrounding-object
trajectory $\Wtraj$.  Applying the
rule-specific risk metric $\mathcal{R}_j$, we define the \emph{robustness risk} of rule $\varphi_j$ as
\begin{equation}
\rhobar{j}(x^e_{0:H}) := -\mathcal{R}_j\!\left(\boldsymbol{Z}_{\varphi_j}(x^e_{0:H})\right),
\quad j = 1,\dots,N.
\label{eq:risk_aggregated_robustness}
\end{equation}
A positive value of
$\rhobar{j}(x^e_{0:H})$ indicates satisfaction of $\varphi_j$ in the risk-aware sense, whereas a negative value indicates a violation.
We denote the combined robustness risk vector,
\begin{equation}
\bar{\rho}(x^e_{0:H}) :=
\bigl(\rhobar{1}(x^e_{0:H}),\dots,\rhobar{N}(x^e_{0:H})\bigr)
\in \mathbb{R}^N.
\label{eq:risk_aggregated_vector}
\end{equation}  
\subsection{Rule Ordering}
\subsubsection{Trajectory Rank under Uncertainty}
To encode the lexicographic priority order among rules in the uncertain
setting, we quantitatively express the rule ordering $\Phi$ using a
rank function following~\cite[Def.~1]{veer2023receding}. Let
$\mathrm{step} : \mathbb{R} \!\to\! \{0,1\}$ map negative real numbers to
$0$ and all other real numbers to $1$. The rank
$r : \mathbb{R}^N \!\to\! \{1, \dots, 2^N\}$ applied to the robustness
risk vector $\bar{\rho}$ from~\eqref{eq:risk_aggregated_vector} is
defined as
\begin{equation}
  r(\bar{\rho}(x^e_{0:H})) := 2^N - \sum_{j=1}^{N} 2^{N-j}\,\mathrm{step}(\rhobar{j}(x^e_{0:H})).
  \label{eq:rank}
\end{equation}
A lower rank value indicates satisfaction of higher-priority rules:
rank $1$ corresponds to all rules being satisfied, while rank $2^N$
corresponds to all rules being violated.
\subsubsection{Rank-Preserving Reward Function}
\label{sec:Rank-Preserving}
 
The rank $r$ from~\cite[Def.~1]{veer2023receding} captures \emph{which} rules are
satisfied but not \emph{how well}. To also reward larger robustness margins
within the same rank, we adapt the scalar reward function from~\cite{veer2023receding} so that it (i) strictly reflects
the lexicographic priority ordering across ranks, and (ii) uses the robustness risk as a continuous tie-breaker within a rank:
\begin{equation}
J\!\left(\bar{\rho}(x^e_{0:H})\right)
  \!:=\! \sum_{j=1}^{N}\!\Bigl(
     a\cdot 2^{N-j+1}\,\mathrm{step}\bigl(\rhobar{j}(x^e_{0:H})\bigr)
     + \tfrac{1}{N}\rhobar{j}(x^e_{0:H})
     \Bigr).
\label{eq:rank_preserving_reward}
\end{equation}
We require $a>0$ to be such that the robustness risk values are bounded as
\begin{equation}
	\rhobar{j}(x^e_{0:H})\in[-a/2,\,a/2].
	\label{eq:boundedness}
\end{equation}

The boundedness condition in~\eqref{eq:boundedness} ensures that the continuous tie-breaking
term $\frac{1}{N}\rhobar{j}$ cannot overturn the discrete priority encoded by
the exponential step terms, and is enforced in practice via
normalization.
 
\begin{remark}
\label{rem:contribution}
Our reward~\eqref{eq:rank_preserving_reward} differs from~\cite[Rem.~1]{veer2023receding}
in that, for large $N$, it does not grow as $a^N$; when $a \gg 2$,
such exponential terms can become numerically large.
\end{remark}
 
\begin{proposition}[Lexicographic Order Preservation under Uncertainty]
\label{prop:rank_preserving}
Let $x^e_{0:H}$ and $x^{e\prime}_{0:H}$ be two ego trajectories whose
robustness risk vectors satisfy~\eqref{eq:boundedness}.
Then the reward $J$ in~\eqref{eq:rank_preserving_reward} preserves the
lexicographic rank ordering:
\[
r\!\left(\bar{\rho}(x^e_{0:H})\right)
  < r\!\left(\bar{\rho}(x^{e\prime}_{0:H})\right)
\;\implies\;
J\!\left(\bar{\rho}(x^e_{0:H})\right)
  > J\!\left(\bar{\rho}(x^{e\prime}_{0:H})\right).
\]
\end{proposition}
 
\begin{proof}
See Appendix~\ref{app:proof}.
\end{proof}

\subsubsection{Planning problem}
The trajectory planning problem introduced in Section~\ref{subsec:prob}
can now be formalized as
\begin{align}
\max_{x^e_{0:H},\,u_{0:H-1}}
\quad & J\!\left(\bar{\rho}(x^e_{0:H})\right) & \label{eq:problem_statement}\\
\text{s.t.}\quad
& x^e_{k+1}=f(x^e_k,u_k),
& k=0,\ldots,H-1,\nonumber\\
& x^e_k\in\mathcal X,
& k=0,\ldots,H,\nonumber\\
& u_k\in\mathcal U,
& k=0,\ldots,H-1,\nonumber\\
& x^e_0 \text{ given}.\nonumber
\end{align}


\section{Approximate Risk-Aware Planning under Rule Ordering}
\subsection{Scenario Tree Approximation}
The problem defined in~\eqref{eq:problem_statement} is challenging to
solve directly because $\bar{\rho}(x^e_{0:H})$ in the objective depends
on the full distribution of the stochastic trajectory $\Wtraj$, which can be continuous.
To obtain a
finite approximation suitable for real-time planning, we
represent the distribution of $\Wtraj$ by a weighted scenario
set~\cite{chen2022branchmpc}, composed of $M$ samples as 
\[
  \mathcal{S} = \bigl\{(\wscen{i},\, p^i)\bigr\}_{i=1}^{M},
\]
where each $\wscen{i} \!=\! (x^{o,i}_0, \dots, x^{o,i}_H) \!\in\! (\Wspace)^{H+1}$
is a realization of the 
surrounding-object trajectory over the
planning horizon with associated probability $p^i \!\ge\! 0$ and
$\sum_{i\!=\!1}^{M} p^i \!=\! 1$.

\noindent\textbf{Per-scenario robustness. }
For a fixed ego trajectory $x^e_{0:H}$ and scenario $\wscen{i}$, the
composed signal realization is
\[
  s_{0:H}^i := \bigl((x^e_0, x^{o,i}_0),\, \dots,\, (x^e_H, x^{o,i}_H)\bigr).
\]
The robustness of rule $\varphi_j$ under scenario $i$ is then the
deterministic scalar
\begin{equation}
  \zij{i}{j} := \rho\bigl(\varphi_j,\, s_{0:H}^i,\, 0\bigr), \quad i = 1,\dots,M,\; j = 1,\dots,N,
  \label{eq:per_scenario_robustness}
\end{equation}
computed from the STL robustness semantics~\eqref{eq:stl_robustness_semantics}.
A positive value $\zij{i}{j} > 0$ indicates that $\varphi_j$ is satisfied
under scenario $i$; a negative value indicates a violation.

\noindent\textbf{Empirical CVaR approximation.}
Under the scenario approximation, the expectation in the CVaR formula
is replaced by a weighted sum over the $M$ scenarios. For a fixed ego
trajectory $x^e_{0:H}$, the empirical robustness risk of
rule $\varphi_j$ is 
\begin{equation}
  \hat \rho_j (x^e_{0:H})
  := - \inf_{\alpha \in \mathbb{R}}\!\left(
      \alpha + \frac{1}{1-\beta}
      \sum_{i=1}^{M} p^i\,
      \bigl(-\zij{i}{j} - \alpha\bigr)^{+}
    \right),
  \label{eq:empirical_cvar}
\end{equation}
for $j \!=\! 1, \dots, N$, where $(\cdot)^{+} \!:=\! \max(\cdot,\, 0)$ and
$\beta \!\in\! [0,1)$ is the CVaR confidence level. This follows directly
from substituting the weighted scenario sum for the expectation
in the CVaR definition of Section~\ref{subsec:risk_metrics}, with loss
values $-\zij{i}{j}$ in place of $\boldsymbol{Z}_{\varphi_j}(x^e_{0:H})$.
A positive value $\hat \rho_j(x^e_{0:H}) > 0$ indicates
satisfaction of $\varphi_j$; a negative value indicates a violation.
We denote the combined empirical robustness risk vector,
\begin{equation}
\hat \rho(x^e_{0:H}) :=
\bigl(\hat \rho_1(x^e_{0:H}),\dots,\hat \rho_N(x^e_{0:H})\bigr)
\in \mathbb{R}^N.
\label{eq:empirical_risk_aggregated_vector}
\end{equation}

\begin{remark}[Scenario Approximation Consistency]
\label{rem:scenario_approx}
The empirical CVaR in~\eqref{eq:empirical_cvar} converges to the true
CVaR of $\boldsymbol{Z}_{\varphi_j}(x^e_{0:H})$ as $M \!\to\! \infty$, provided the
scenarios are drawn i.i.d.\ from the true distribution of
$\Wtraj$~\cite{shapiro2009lectures}. The approximation quality depends
on how well the finite scenario set $\mathcal{S}$ captures the support
and weights of the true distribution. 
\end{remark}


\subsection{Model Predictive Path Integral (MPPI) Control}

To solve~\eqref{eq:problem_statement}, we employ
MPPI~\cite{williams2016aggressive,williams2017model}, a sampling-based
stochastic optimal control method particularly well suited to our setting.
At each receding-horizon step, MPPI draws $V$ independent control
perturbations, rolls each out through the ego dynamics~\eqref{eq:ego_dynamics},
scores the resulting trajectories via a cost function, and returns a
single control update via importance-weighted averaging, with temperature
parameter $\lambda > 0$ controlling the selectivity of the weighting.
The reward $J$ in~\eqref{eq:rank_preserving_reward} is inherently non-smooth:
STL robustness is built from nested $\min$/$\max$ operators, and the
rank-preserving reward contains a discontinuous step function. While smooth approximations of STL robustness exist~\cite{gilpin2021smooth}, they introduce approximation error that may both distort the
lexicographic ordering and compromise the soundness property of STL
robustness~\cite{donze2010robust}. MPPI avoids this issue entirely, as it requires
neither differentiability nor convexity and evaluates each rollout exactly. The $V$ control perturbations are
independently sampled and rolled out (Alg.~\ref{alg:mppi}, lines~5--6),
scored via the rank-preserving reward (lines~7--9), and combined into a
single control update via importance-weighted averaging (lines~11--12).
The central design choice lies in the rollout cost at
Alg.~\ref{alg:mppi}, line~9: by setting
$C^{(\ell)} = -J(\hat{\rho}^{(\ell)})$, the MPPI objective
directly encodes the priority-ordered robustness risk vector.

\begin{algorithm}[t]
\caption{MPPI Receding-Horizon Planner}
\label{alg:mppi}
\begin{algorithmic}[1]
\STATE \textbf{Input:} initial state $x^e_0$, nominal controls $u_{0:H-1}$,
       scenario set $\mathcal{S}$, number of samples $V$, temperature $\lambda$
\STATE \textbf{Output:} updated nominal control sequence $u_{0:H-1}$
\WHILE{planning is active}
  \FOR{each sample $\ell \!=\! 1, \dots, V$}
    \STATE Draw perturbations $\epsilon_k^{(\ell)} \sim \mathcal{N}(0,\Sigma)$,\;
           $k \!=\! 0,\dots,H-1$
    \STATE Roll out $v_k^{(\ell)} \!=\! u_k + \epsilon_k^{(\ell)}$
           through~\eqref{eq:ego_dynamics} $\;\!\to\!\;$ $(x^e_{0:H})^{(\ell)}$
    \STATE Evaluate per-scenario robustness $z^i_j$
           via~\eqref{eq:per_scenario_robustness}
    \STATE Compute the empirical robustness risk vector: $\hat \rho((x^e_{0:H})^{(\ell)})$ 
           via~\eqref{eq:empirical_risk_aggregated_vector}
  \STATE Compute rollout cost
        $C^{(\ell)} \leftarrow
        -J\!\left(\hat{\rho}\!\left((x^e_{0:H})^{(\ell)}\right)\right)$
  \ENDFOR
  \STATE Compute importance weights $\;\omega^{(\ell)} \propto
         \exp\!\bigl(-C^{(\ell)}/\lambda\bigr)$
  \STATE Update $\;u_k \leftarrow u_k + \sum_{\ell\!=\!1}^{V}
         \omega^{(\ell)}\epsilon_k^{(\ell)}$,\; $k\!=\!0,\dots,H-1$
  \STATE Apply $u_0$; shift horizon and warm-start $u_{0:H-1}$
\ENDWHILE
\end{algorithmic}
\end{algorithm}

\section{Simulation Results and Validation}
We validate our proposed approach on two autonomous driving scenarios: a
highway overtaking maneuver with a multimodal front-vehicle cut-in (Use
Case~A), and an unsignalized intersection turn decision under multimodal
oncoming-vehicle motion hypotheses (Use Case~B).

\subsection{Implementation Details}

\noindent\textbf{Dynamics.}
The ego vehicle is modeled by the kinematic bicycle model~\cite{kong2015kinematic}, a standard choice for autonomous driving~\cite{williams2017model}. The state is $x^e_k \!=\! (p^x_k,\, p^y_k,\, \psi_k,\, v_k)$, collecting the planar position, heading angle, and longitudinal speed, while the control input $u_k \!=\! (a_k,\, \delta_k)$ consists of the longitudinal acceleration and the front-wheel steering angle. The continuous-time kinematics
\begin{equation}
\dot{p}^x = v\cos\psi,\quad
\dot{p}^y = v\sin\psi,\quad
\dot{\psi} = \frac{v}{\ell}\tan\delta,\quad
\dot{v} = a,
\label{eq:bicycle}
\end{equation}
are discretized with sampling time $\Delta t$ via forward Euler integration.
We denote the ego planar position as $p_k \!:=\! (p^x_k,\, p^y_k)$ and
the planar position of the surrounding object as
$p^o_k \!:=\! (x^{o,x}_k,\, x^{o,y}_k)$ at step $k$.

\noindent\textbf{STL Rules.}
Table~\ref{tab:specifications} lists the STL specifications used across the two use cases. $\varphi_{\mathrm{safe}}$ enforces a minimum clearance $d_{\mathrm{safe}}$ from the surrounding object over the planning horizon, and $\varphi_{\mathrm{goal}}$ requires the ego vehicle to reach a goal region of radius $d_{\mathrm{goal}}$ around $g$ within the horizon. The Not-cross-dashed-lane-line rule $\varphi_{\mathrm{dash}}$ confines the lateral position to the lane boundaries, preventing dashed-lane crossings.
\begin{table}[!t]
\caption{STL rule specifications.}
\label{tab:specifications}
\centering
\footnotesize
\setlength{\tabcolsep}{3pt}
\renewcommand{\arraystretch}{1.4}
\begin{tabular}{@{}p{0.38\columnwidth} p{0.58\columnwidth}@{}}
\toprule
\textbf{Rule} & \textbf{STL Formula} \\
\midrule
Safety
& $\varphi_{\mathrm{safe}} \!=\! G_{[0,H]}(\|p_k-p^{o}_k\| \!\ge\! d_{\mathrm{safe}})$ \\[4pt]
Goal
& $\varphi_{\mathrm{goal}} = F_{[0,H]}(\|p_k-g\| \le d_{\mathrm{goal}})$ \\[4pt]
Not-cross-dashed-lane-line
& $\varphi_{\mathrm{dash}} = G_{[0,H]}(y_{\mathrm{low}}\le y_k \le y_{\mathrm{up}})$ \\
\bottomrule
\end{tabular}
\end{table}

\noindent\textbf{MPPI Parameters.}
Use Case~A uses a planning horizon of $H=16$, a sampling time of
$\Delta t=0.2\,\mathrm{s}$, a wheelbase of $\ell=2.7\,\mathrm{m}$,
$V_A=220$ sampled trajectories, temperature $\lambda=0.1$, and Gaussian
control perturbations with
$\sigma_a=1.1\,\mathrm{m/s^2}$ and
$\sigma_\delta=0.08\,\mathrm{rad}$. Use Case~B uses $H=24$,
$\Delta t=0.2\,\mathrm{s}$, and $\ell=1.2\,\mathrm{m}$, with
$V_B=240$ trajectories,
and $\lambda=0.08$. The perturbation standard deviations are
$\sigma_a=8.0\,\mathrm{m/s^2}$ and
$\sigma_\delta=0.055\,\mathrm{rad}$.
Both use cases employ the rank-preserving priority parameter $a=2.01$.

\noindent\textbf{Computational Platform.}
All experiments are run on a laptop equipped with an AMD Ryzen 9 8940HX CPU and an NVIDIA GeForce RTX~5060 GPU. The MPPI sample rollouts
are parallelized on the GPU following the implementation of the
MPPI-Generic CUDA library~\cite{vlahov2024mppi}, which assigns each of the
$V$ independent trajectory rollouts to a separate GPU thread, fully exploiting the parallel structure of MPPI.

\subsection{Use Case A: Highway Overtaking Under Uncertainty}
\label{subsec:highway}

\noindent\textbf{Scenario description.}
As illustrated in Fig.~\ref{fig:use_case_highway}, the ego vehicle starts in
the middle lane of a three-lane highway with the objective of reaching a
longitudinal goal at $x = 30.0\,\mathrm{m}$ in the same lane. A front vehicle initially occupying the lower lane may execute
a cut-in maneuver into the middle lane. The uncertainty over its future motion
is represented in this use case by five scenarios: one in which it keeps its
lane and four cut-in profiles of increasing aggressiveness. The choice of five scenarios is not restrictive; the framework supports any number of scenarios. The lateral position of each
cut-in profile evolves as a logistic sigmoid from the lower lane center
$(-1.6\,\mathrm{m})$ to the middle lane center $(1.6\,\mathrm{m})$. The keep-lane scenario $W^5$ is assigned probability $0.02$, while the
four cut-in profiles $W^1$, $W^2$, $W^3$, and $W^4$ are assigned probabilities $0.55$, $0.25$, $0.12$, and
$0.06$, respectively. In
practice, such a distribution can be produced by a multimodal motion prediction
model such as MTP~\cite{cui2019multimodal},
MultiPath~\cite{chai2020multipath}, or
CoverNet~\cite{phanminh2020covernet}. The safety distance and goal
tolerance are set to $d_{\mathrm{safe}}=3.0\,\mathrm{m}$ and
$d_{\mathrm{goal}}=1.0\,\mathrm{m}$, respectively. For the
Not-cross-dashed-lane-line rule $\varphi_{\mathrm{dash}}$, the lateral bounds are
$y_{\mathrm{low}}=0.1\,\mathrm{m}$ and
$y_{\mathrm{up}}=3.1\,\mathrm{m}$.


\noindent\textbf{Lexicographic and weighted-CVaR rewards.}
We compare two reward formulations using the same scenario distribution and
CVaR level $\beta=0.50$. The green ego uses the rank-preserving lexicographic
reward~\eqref{eq:rank_preserving_reward} with
$\varphi_{\mathrm{safe}} \succ \varphi_{\mathrm{goal}} \succ
\varphi_{\mathrm{dash}}$. The blue ego instead uses a weighted scalar reward.
The robustness of each rule $\varphi_j$ is first aggregated over the discrete
scenario distribution using lower-tail CVaR,
$\bar{\rho}^{\beta}_j=-\mathrm{CVaR}_{\beta}(-\rho_j)$. The weighted-CVaR reward is
then
\begin{equation}
  R_{\mathrm{wCVaR}}
  = w_{\mathrm{safe}}\bar{\rho}^{\beta}_{\mathrm{safe}}
  + w_{\mathrm{goal}}\bar{\rho}^{\beta}_{\mathrm{goal}}
  + w_{\mathrm{dash}}\bar{\rho}^{\beta}_{\mathrm{dash}}.
  \label{eq:weighted_cvar_reward}
\end{equation}
For this experiment, we set $w_{\mathrm{safe}}=10$,
$w_{\mathrm{goal}}=5$, and $w_{\mathrm{dash}}=8$. In this scenario, choosing
$w_{\mathrm{safe}}$ sufficiently large relative to the other weights can
produce behavior similar to that of the safety-first
lexicographic reward. The lexicographic reward directly encodes the priority and
thus avoids tuning these weights.

\begin{figure}[htbp]
    \centering
    \includegraphics[width=0.48\textwidth]{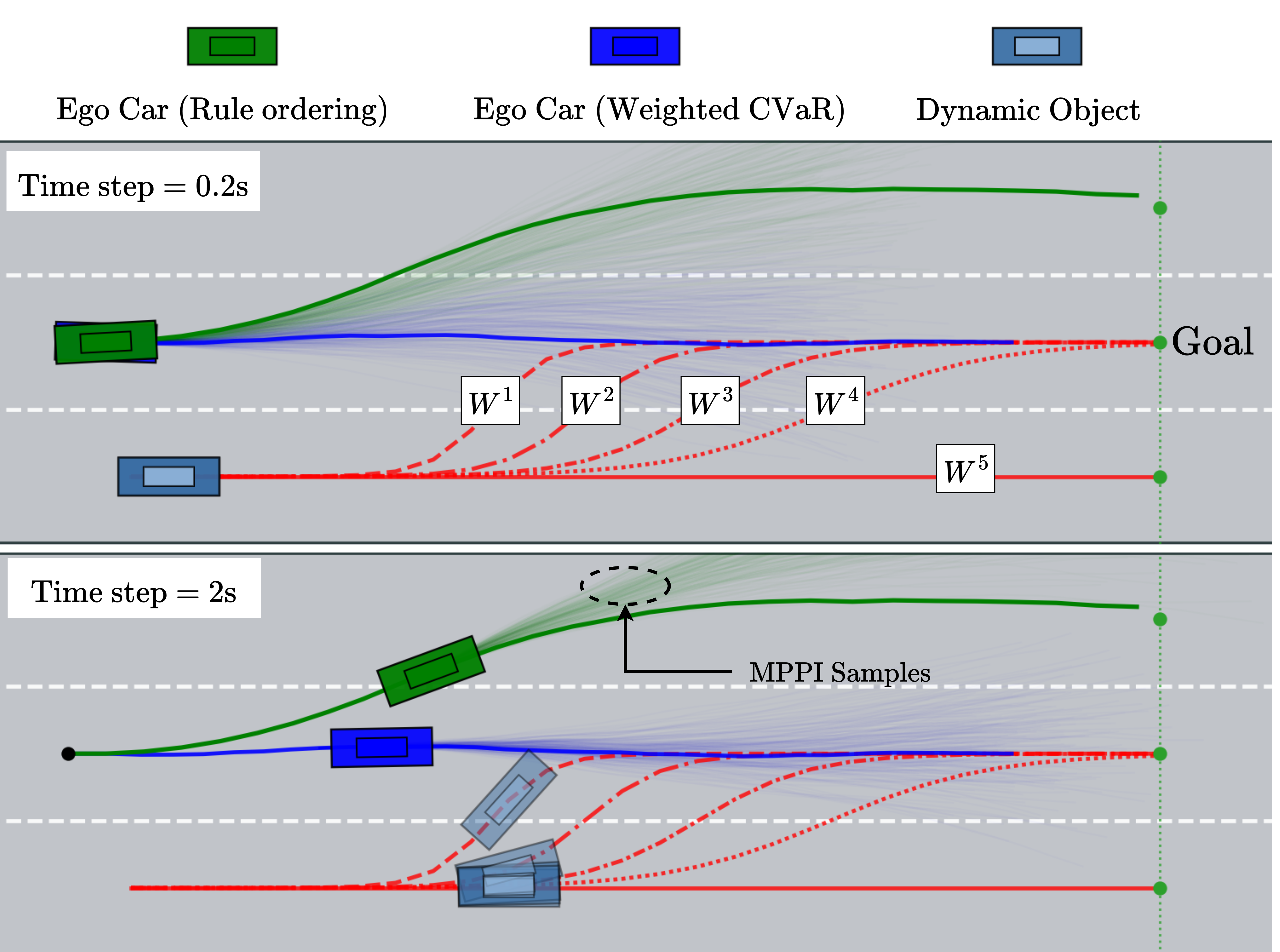}
    \caption{Highway overtaking scenario comparing the rank-preserving
    lexicographic reward (green) with the weighted-CVaR reward in
    \eqref{eq:weighted_cvar_reward} (blue) at $t=0.2\,\mathrm{s}$ and
    $t=2\,\mathrm{s}$. Red lines show the predicted trajectories of the
    surrounding vehicle.}
    \label{fig:use_case_highway}
\end{figure}

\noindent\textbf{Results.}
Fig.~\ref{fig:use_case_highway} shows two snapshots of the executed
trajectories at $t=0.2\,\mathrm{s}$ and $t=2\,\mathrm{s}$. The green ego uses
the safety-first lexicographic ordering
$\varphi_{\mathrm{safe}} \succ \varphi_{\mathrm{goal}} \succ
\varphi_{\mathrm{dash}}$. As described in Alg.~\ref{alg:mppi}, the MPPI
planner steers its sampling distribution toward trajectories with high reward
according to~\eqref{eq:rank_preserving_reward}. By $t=2\,\mathrm{s}$, the
green ego has crossed the upper dashed line, thereby violating
$\varphi_{\mathrm{dash}}$, while maintaining a distance of more than
$d_{\mathrm{safe}}$ from all predicted cut-in trajectories.

The blue ego uses the weighted-CVaR reward
in~\eqref{eq:weighted_cvar_reward} and remains in the middle lane, thereby
satisfying $\varphi_{\mathrm{dash}}$. The computation time per planning step is
$0.77\,\mathrm{ms}$, enabled by GPU parallelization across the 220 sampled
rollouts.

\subsection{Use Case B: Intersection Turn Decision}
\label{subsec:intersection}

\noindent\textbf{Scenario description.}
The second use case considers an ego vehicle approaching an unsignalized intersection from the south and deciding whether to initiate a left turn or wait for an oncoming vehicle. The ego vehicle receives two constant-velocity motion hypotheses of the oncoming vehicle, generated either through communication or by its own multimodal predictor. The first is a yielding hypothesis, shown on the left side of Fig.~\ref{fig:intersection}, with a speed of $v_{\mathrm{yield}}=0.35\,\mathrm{m/s}$. The second is a proceed-through hypothesis, shown on the right side of the same figure, with a speed of $v_{\mathrm{proceed}}=1.2\,\mathrm{m/s}$.
The ego vehicle starts at
$p_0=(0.9,-3.0)\,\mathrm{m}$ with zero speed.
The left-turn goal is at $g=(-3.8,0.9)\,\mathrm{m}$. The oncoming vehicle starts
at $p^o_0=(-0.9,4.4)\,\mathrm{m}$. For this use case, safety is assigned higher priority than goal progress, giving the lexicographic ordering
$\varphi_{\mathrm{safe}}\succ\varphi_{\mathrm{goal}}$.
The safety distance, goal tolerance, and CVaR confidence level are set to
$d_{\mathrm{safe}}=1.35\,\mathrm{m}$,
$d_{\mathrm{goal}}=0.8\,\mathrm{m}$, and
$\beta=0.8$, respectively.

\begin{figure}[htbp]
    \centering
    \includegraphics[width=\columnwidth]{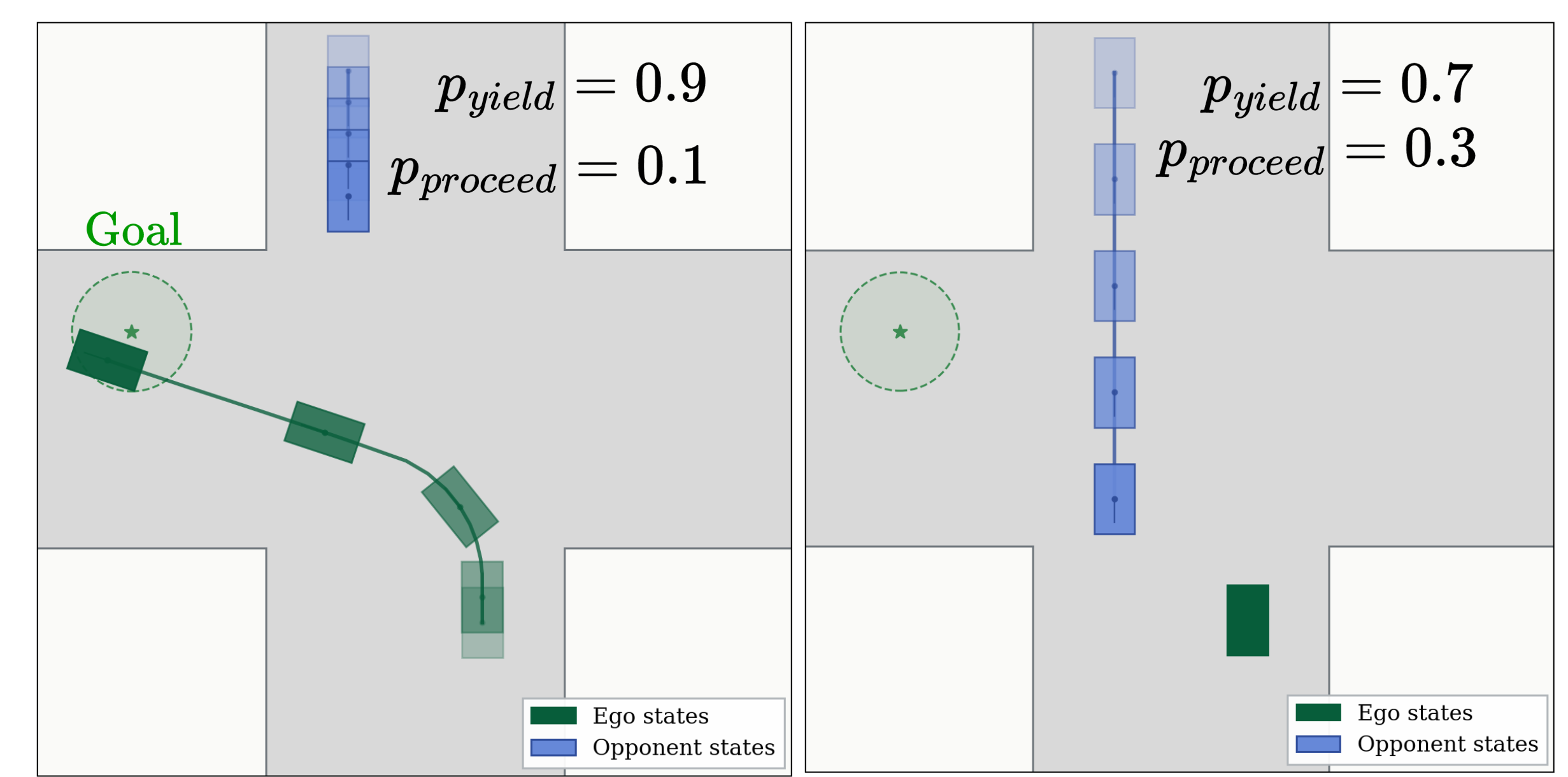}
\caption{Intersection scenario under two probability distributions over the oncoming vehicle's motion hypotheses. In the left panel, $(p_{\mathrm{yield}},p_{\mathrm{proceed}})=(0.9,0.1)$, and the ego vehicle proceeds with the left turn toward the goal region. In the right panel, $(p_{\mathrm{yield}},p_{\mathrm{proceed}})=(0.7,0.3)$; the increased probability that the oncoming vehicle proceeds through the intersection causes the safety-prioritized planner to yield.}
    \label{fig:intersection}
\end{figure}

\noindent\textbf{Results.}
Fig.~\ref{fig:intersection} illustrates how the uncertainty distribution over the possible motions of the oncoming vehicle influences the selected ego maneuver. The left side of the figure shows the case in which the ego vehicle proceeds with the left turn, whereas the right side shows the case in which it yields to the oncoming vehicle. When the hypothesis that the oncoming vehicle will yield is highly probable, with $(p_{\mathrm{yield}},p_{\mathrm{proceed}})=(0.9,0.1)$, the ego vehicle commits to the left turn and successfully reaches the goal region. In contrast, when the hypothesis that the oncoming vehicle will proceed through the intersection becomes more likely, with $(p_{\mathrm{yield}},p_{\mathrm{proceed}})=(0.7,0.3)$, the ego vehicle selects a yielding trajectory. Although the probability that the oncoming vehicle will proceed is only $0.3$, committing to the turn could result in a collision. Because the safety specification has higher priority in the lexicographic ordering, the ego vehicle sacrifices progress toward the goal to satisfy the higher-priority safety requirement.

The computation time for a single planning step is $2.49\,\mathrm{ms}$. This time includes the generation of the sampled rollouts, bicycle-model dynamics propagation, robustness evaluation, CVaR aggregation, reward computation, and the nominal-control update.

\section{Conclusion}
We presented a trajectory planning framework for autonomous vehicles
that enforces lexicographically ordered STL specifications under
multimodal uncertainty. By combining CVaR-based robustness risk with
a rank-preserving reward, our approach ensures that higher-priority
rules such as safety strictly dominate lower-priority objectives
in the risk-aware sense. The resulting planning problem is solved efficiently via an MPPI-based
receding-horizon planner, which we validate on a highway overtaking scenario and an
unsignalized intersection turn-decision scenario. Several directions remain open for future work, including extending the priority structure to encode ethical considerations, incorporating intention-aware behavioral prediction, and validating on real-world driving benchmarks.

\small
\setlength{\baselineskip}{0.9\baselineskip}

\bibliographystyle{IEEEtran}
\bibliography{ref}

\appendix[Proof of Proposition~\ref{prop:rank_preserving}]
\label{app:proof}

The proof follows the structure of~\cite[Appendix]{veer2023receding},
adapted to the reward $J$ in~\eqref{eq:rank_preserving_reward} with base~$2$.

\textit{Setup.}
Let $\bar{\rho}$ and $\bar{\rho}'$ denote the robustness risk vectors of
$x^e_{0:H}$ and $x^{e\prime}_{0:H}$, respectively, with
$r(\bar{\rho}) < r(\bar{\rho}')$.
By definition of rank, there exists a highest-priority index at which
$\bar{\rho}$ satisfies a rule that $\bar{\rho}'$ does not. Define
\[
k := \min\bigl\{j \mid
     \mathrm{step}(\rhobar{j}) > \mathrm{step}(\rhobar{j}'),\;
     j\in\{1,\dots,N\}\bigr\},
\]
so that $\mathrm{step}(\rhobar{k})\!=\!1$ and $\mathrm{step}(\rhobar{k}')\!=\!0$.
By minimality of $k$, the step values agree for all $j < k$, i.e.,
$\mathrm{step}(\rhobar{j}) \!=\! \mathrm{step}(\rhobar{j}')$ for $j \!=\! 1,\dots,k-1$.

\textit{Decomposition.}
Define $b \!:=\! \sum_{j\!=\!1}^{k-1} a\cdot 2^{N-j+1}\,\mathrm{step}(\rhobar{j})$.
Since the step values agree for $j < k$, this constant $b$ is identical
for both $\bar{\rho}$ and $\bar{\rho}'$, and the rewards decompose as
\begin{align*}
J(\bar{\rho})
  &= b + a\cdot 2^{N-k+1}
   + \sum_{j=k+1}^{N} a\cdot 2^{N-j+1}\,\mathrm{step}(\rhobar{j})
   + \tfrac{1}{N}\sum_{j=1}^{N}\rhobar{j},\\
J(\bar{\rho}')
  &= b + 0
   + \sum_{j=k+1}^{N} a\cdot 2^{N-j+1}\,\mathrm{step}(\rhobar{j}')
   + \tfrac{1}{N}\sum_{j=1}^{N}\rhobar{j}'.
\end{align*}
It remains to show $J(\bar{\rho}') < J(\bar{\rho})$,
which we establish via two claims.

\textbf{Claim 1:}
$\displaystyle\sum_{j\!=\!k+1}^{N} a\cdot 2^{N-j+1}\,\mathrm{step}(\rhobar{j}')
 < a\cdot 2^{N-k+1} - a.$

Since $\mathrm{step}(\cdot) \!\le\! 1$, we bound the sum by dropping the
step function and evaluating the geometric series exactly:
\[
\sum_{j=k+1}^{N} a\cdot 2^{N-j+1}\,\mathrm{step}(\rhobar{j}')
\;\leq\;
a\sum_{j=k+1}^{N} 2^{N-j+1}
\;=\; a\bigl(2^{N-k+1} - 2\bigr).
\]
Since $a > 0$, we have $a(2^{N-k+1}-2) < a(2^{N-k+1}-1)$,
which completes the proof of Claim~1.

\textbf{Claim 2:}
$\displaystyle\tfrac{1}{N}\sum_{j\!=\!1}^{N}\rhobar{j}' - a
 \;\leq\; \tfrac{1}{N}\sum_{j\!=\!1}^{N}\rhobar{j}.$

By the boundedness assumption~\eqref{eq:boundedness}, each
$\rhobar{j}, \rhobar{j}' \!\in\! [-a/2,\, a/2]$. Therefore
$\tfrac{1}{N}\sum_{j}\rhobar{j}' \leq a/2$ and
$\tfrac{1}{N}\sum_{j}\rhobar{j} \geq -a/2$,
so the difference satisfies
$\tfrac{1}{N}\sum_{j}\rhobar{j}' - \tfrac{1}{N}\sum_{j}\rhobar{j} \leq a$,
which is exactly Claim~2.

\textit{Conclusion.}
Applying Claim~1 to bound the lower-priority step terms of $J(\bar{\rho}')$,
and then Claim~2 to handle the tie-breaking terms, gives
\begin{align*}
J(\bar{\rho}')
  &\leq b + a\bigl(2^{N-k+1}-2\bigr)
         + \tfrac{1}{N}\sum_{j=1}^{N}\rhobar{j}'
         & &\text{(Claim 1)}\\
  &\leq b + a\bigl(2^{N-k+1}-2\bigr)
         + \tfrac{1}{N}\sum_{j=1}^{N}\rhobar{j} + a
         & &\text{(Claim 2)}\\
  &=    b + a\cdot 2^{N-k+1} - a
         + \tfrac{1}{N}\sum_{j=1}^{N}\rhobar{j}\\
  &<    b + a\cdot 2^{N-k+1}
         + \tfrac{1}{N}\sum_{j=1}^{N}\rhobar{j}
  \;=\; J(\bar{\rho}),
\end{align*}
which completes the proof.\qed

\end{document}